\documentclass[11pt]{article}

\usepackage{vitercik}
\usepackage[colorlinks=true, citecolor=teal, linkcolor=black]{hyperref}
\newcommand{\score}{\texttt{score}}
\newcommand{\tree}{\cT}
\newcommand{\lf}{\left\lfloor}
\newcommand{\rf}{\right\rfloor}
\newcommand{\LP}{\mathsf{LP}}
\usepackage{subcaption}

\title{Sample Complexity of Tree Search Configuration:\\ Cutting Planes and Beyond}
\author{Maria-Florina Balcan\thanks{Computer Science Department, Machine Learning Department, Carnegie Mellon University. \texttt{ninamf@cs.cmu.edu}} \and Siddharth Prasad\thanks{Computer Science Department, Carnegie Mellon University. \texttt{sprasad2@cs.cmu.edu}} \and Tuomas Sandholm\thanks{Computer Science Department, Carnegie Mellon University, Optimized Markets, Inc., Strategic Machine, Inc., Strategy Robot, Inc. \texttt{sandholm@cs.cmu.edu}} \and Ellen Vitercik\thanks{Computer Science Department, Carnegie Mellon University. \texttt{vitercik@cs.cmu.edu}}}

\begin{document}

\maketitle

\begin{abstract}
    Cutting-plane methods have enabled remarkable successes in integer programming over the last few decades. State-of-the-art solvers integrate a myriad of cutting-plane techniques to speed up the underlying tree-search algorithm used to find optimal solutions. In this paper we prove the first guarantees for learning high-performing cut-selection policies tailored to the instance distribution at hand using samples. We first bound the sample complexity of learning cutting planes from the canonical family of Chv\'{a}tal-Gomory cuts. Our bounds handle any number of waves of any number of cuts and are fine tuned to the magnitudes of the constraint coefficients. Next, we prove sample complexity bounds for more sophisticated cut selection policies that use a combination of scoring rules to choose from a family of cuts. Finally, beyond the realm of cutting planes for integer programming, we develop a general abstraction of tree search that captures key components such as node selection and variable selection. For this abstraction, we bound the sample complexity of learning a good policy for building the search tree.
\end{abstract}

\section{Introduction}

Integer programming is one of the most broadly-applicable tools in computer science, used to formulate problems from operations research (such as routing, scheduling, and pricing), machine learning (such as adversarially-robust learning, MAP estimation, and clustering), and beyond. \emph{Branch-and-cut (B\&C)} is the most widely-used algorithm for solving integer programs (IPs). B\&C is highly configurable, and with a deft configuration, it can be used to solve computationally challenging problems. Finding a good configuration, however, is a notoriously difficult problem.

We study machine learning approaches to configuring policies for selecting \emph{cutting planes}, which have an enormous impact on B\&C's performance~\cite{Nemhauser99:Integer,Balas96:Mixed,Chvatal73:Edmonds,Gomory58:Outline,Cornuejols01:Elementary}. At a high level, B\&C works by recursively partitioning the IP's feasible region, searching for the locally optimal solution within each set of the partition, until it can verify that it has found the globally optimal solution. An IP's feasible region is defined by a set of linear inequalities $A \vec{x} \leq \vec{b}$ and integer constraints $\vec{x} \in \Z^n$, where $n$ is the number of variables. By dropping the integrality constraints, we obtain the \emph{linear programming (LP) relaxation} of the IP, which can be solved efficiently. A cutting plane is a carefully-chosen linear inequality $\vec{\alpha}^T\vec{x} \leq \beta$ which refines the LP relaxation's feasible region without separating any integral point. Intuitively, a well-chosen cutting plane will remove a large portion of the LP relaxation's feasible region, speeding up the time it takes B\&C to find the optimal solution to the original IP.
Cutting plane selection is a crucial task, yet it is challenging because many cutting planes and cut-selection policies have tunable parameters, and the best configuration depends intimately on the application domain.

We provide the first provable guarantees for learning high-performing cutting planes and cut-selection policies, tailored to the application at hand. We model the application domain via an unknown, application-specific distribution over IPs, as is standard in the literature on using machine learning for integer programming~\cite[e.g.,][]{Sandholm13:Very-Large-Scale,Hutter09:Paramils,Leyton-Brown09:Empirical,Kadioglu10:ISAC,Xu11:Hydra-MIP}. For example, this could be a distribution over the routing IPs that a shipping company must solve day after day. The learning algorithm's input is a training set sampled from this distribution. The goal is to use this training set to learn cutting planes and cut-selection policies with strong future performance on problems from the same application but which are not already in the training set---or more formally, strong expected performance.

\subsection{Summary of main contributions and overview of techniques}

As our first main contribution, we bound the \emph{sample complexity} of learning high-performing cutting planes. Fixing a family of cutting planes, these guarantees bound the number of samples sufficient to ensure that for any sequence of cutting planes from the family, its average performance over the samples is close to its expected performance. We measure performance in terms of the size of the search tree B\&C builds. Our guarantees apply to the parameterized family of \emph{Chv\'{a}tal-Gomory (CG) cuts}~\citep{Chvatal73:Edmonds,Gomory58:Outline}, one of the most widely-used families of cutting planes.

The overriding challenge is that to provide guarantees, we must analyze how the tree size changes as a function of the cut parameters. This is a sensitive function: slightly shifting the parameters can cause the tree size to shift from constant to exponential in the number of variables.
Our key technical insight is that as the parameters vary, the entries of the cut (i.e., the vector $\vec{\alpha}$ and offset $\beta$ of the cut $\vec{\alpha}^T\vec{x} \leq \beta$) are multivariate polynomials of bounded degree. The number of terms defining the polynomials is exponential in the number of parameters, but we show that the polynomials can be embedded in a space with dimension sublinear in the number of parameters. This insight allows us to better understand tree size as a function of the parameters. We then leverage results by \citet{Balcan21:How} that show how to use structure exhibited by dual functions (measuring an algorithm's performance, such as its tree size, as a function of its parameters) to derive sample complexity bounds.

Our second main contribution is a sample complexity bound for learning cut-selection policies, which allow B\&C to adaptively select cuts as it solves the input IP. These cut-selection policies assign a number of real-valued scores to a set of cutting planes and then apply the cut that has the maximum weighted sum of scores.
\begin{figure}

     \centering
     \begin{subfigure}[b]{0.47\textwidth}
        \captionsetup{justification=centering}
         \centering
         \includegraphics[scale=0.48]{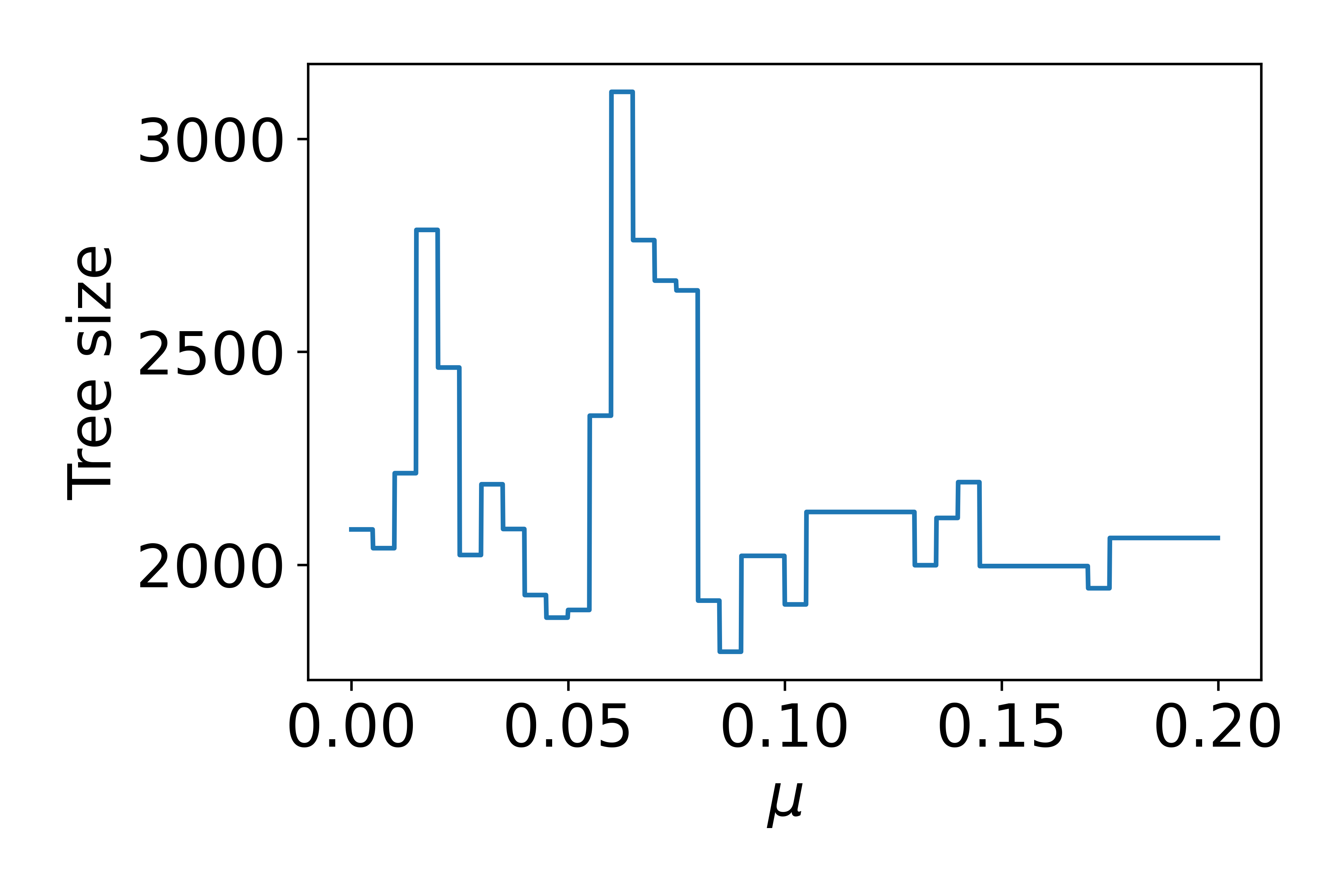}
     \end{subfigure}
     \hfill
     \begin{subfigure}[b]{0.47\textwidth}
        \captionsetup{justification=centering}
         \centering
         \includegraphics[scale=0.48]{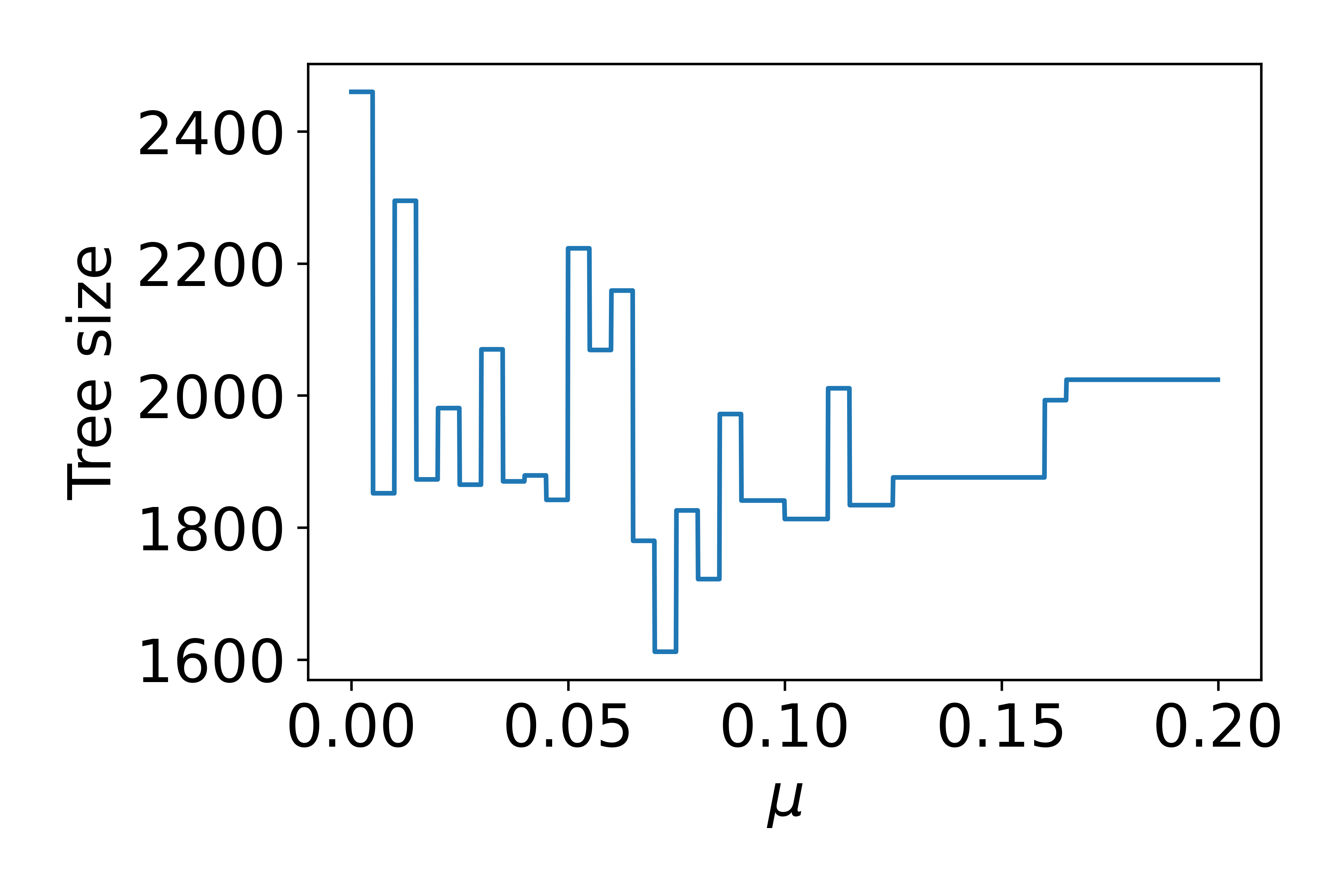}
     \end{subfigure}
    \caption{Two examples of tree size as a function of a SCIP cut-selection parameter $\mu$ (the directed cutoff distance weight, defined in Section~\ref{sec:formulation}) on IPs generated from the Combinatorial Auctions Test Suite~\citep{Leyton00:Toward} (the ``regions'' generator with 600 bids and 600 goods). SCIP~\citep{Gamrath20:SCIP} is the leading open-source IP solver.}
    \label{fig:dual}
\end{figure}
Tree size is a volatile function of these weights, though we prove that it is piecewise constant, as illustrated in Figure~\ref{fig:dual}, which allows us to prove our sample complexity bound.

Finally, as our third main contribution, we provide guarantees 
for tuning weighted combinations of scoring rules for other aspects of tree search beyond cut selection, including node and variable selection.
We prove that there is a set of hyperplanes splitting the parameter space into regions such that if tree search uses any configuration from a single region, it will take the same sequence of actions. This structure allows us to prove our sample complexity bound. This is the first paper to provide guarantees for tree search configuration that apply simultaneously to multiple different aspects of the algorithm---prior research was specific to variable selection~\citep{Balcan18:Learning}.

\subsection{Related work}

\paragraph{Applied research on tree search configuration.} Over the past decade, a substantial literature has developed on the use of machine learning for integer programming and tree search~\citep[e.g.,][]{Sandholm13:Very-Large-Scale,Xu08:Satzilla,Leyton-Brown09:Empirical,Kadioglu10:ISAC,Xu11:Hydra-MIP,Khalil16:Learning,Alvarez17:Machine,DiLiberto16:Dash,Xia18:Learning, Liang16L:earning,Lagoudakis01:Learning,Sabharwal12:Guiding,He14:Learning,Khalil17:Learning,Balcan20:Refined,Lodi17:Learning,Bengio20:Machine}. This has included research that improves specific aspects of B\&C such as variable selection~\citep{Khalil16:Learning,Alvarez17:Machine,DiLiberto16:Dash,Xia18:Learning, Liang16L:earning,Lagoudakis01:Learning}, node selection~\citep{Sabharwal12:Guiding,He14:Learning}, and heuristic scheduling~\citep{Khalil17:Learning}. These papers are applied, whereas we focus on providing theoretical guarantees.

With respect to cutting plane selection, the focus of this paper, \citet{Sandholm13:Very-Large-Scale} uses machine learning techniques to customize B\&C for combinatorial auction winner determination, including cutting plane selection.
\citet{Tang20:Reinforcement} study machine learning approaches to cutting plane selection. They formulate this problem as a reinforcement learning problem and show that their approach can outperform human-designed heuristics for a variety of tasks. Meanwhile, the focus of our paper is to provide the first provable guarantees for cutting plane selection via machine learning.

\citet{Ferber20:Mipaal} study a problem where the IP objective vector $\vec{c}$ is unknown, but an estimate $\hat{\vec{c}}$ can be obtained from data. Their goal is to optimize the quality of the solutions obtained by solving the IP defined by $\hat{\vec{c}}$. They do so by formulating the IP as a differentiable layer in a neural network. The nonconvex nature of the IP does not allow for straightforward gradient computations, so they obtain a continuous surrogate using cutting planes.

\paragraph{Provable guarantees for algorithm configuration.} \citet{Gupta17:PAC} initiated the study of sample complexity bounds for algorithm configuration. A chapter by \citet{Balcan20:Data} provides a comprehensive survey. In research most related to ours, \citet{Balcan18:Learning} provide sample complexity bounds for learning tree search \emph{variable selection policies (VSPs)}. They prove their bounds by showing that for any IP, hyperplanes partition the VSP parameter space into regions where the B\&C tree size is a constant function of the parameters. The analysis in this paper requires new techniques because although we prove that the B\&C tree size is a piecewise-constant function of the CG cutting plane parameters, the boundaries between pieces are far more complex than hyperplanes: they are hypersurfaces defined by multivariate polynomials.

\citet{Kleinberg17:Efficiency,Kleinberg19:Procrastinating} and \citet{Weisz18:LeapsAndBounds,Weisz19:CapsAndRuns}
design configuration procedures for runtime minimization that come with theoretical guarantees.
Their algorithms are designed for the case where there are finitely-many parameter settings to choose from (although they are still able to provide guarantees for infinite parameter spaces by running their procedure on a finite sample of configurations; \citet{Balcan18:Learning,Balcan20:Learning} analyze when discretization approaches can and cannot be gainfully employed). In contrast, our guarantees are designed for infinite parameter spaces.

\section{Problem formulation}\label{sec:formulation}

In this section we give a more detailed technical overview of branch-and-cut, as well as an overview of the tools from learning theory we use to prove sample complexity guarantees.

\subsection{Branch-and-cut}\label{sec:BandC}
We study integer programs (IPs) in canonical form given by \begin{equation}\max\left\{\vec{c}^T\vec{x}:A\vec{x}\le\vec{b},\vec{x}\ge 0, \vec{x}\in\Z^n\right\},\label{eq:IP}\end{equation} where $A\in\Z^{m\times n}$, $\vec{b}\in\Z^m$, and $\vec{c} \in \R^n$.
Branch-and-cut (B\&C) works by recursively partitioning the input IP's feasible region, searching for the locally optimal solution within each set of the partition until it can verify that it has found the globally optimal solution. It organizes this partition as a search tree, with the input IP stored at the root. It begins by solving the LP relaxation of the input IP; we denote the solution as $\vec{x}^*_{\LP} \in \R^n$.  If $\vec{x}^*_{\LP}$ satisfies the IP's integrality constraints $\left(\vec{x}^*_{\LP}\in\Z^n\right)$, then the procedure terminates---$\vec{x}^*_{\LP}$ is the globally optimal solution. Otherwise, it uses a \emph{variable selection policy} to choose a variable $x[i]$. In the left child of the root, it stores the original IP with the additional constraint that $x[i] \leq \lfloor x^*_{\LP}[i]\rfloor$, and in the right child, with the additional constraint that $x[i] \geq \lceil x^*_{\LP}[i]\rceil$. It then uses a \emph{node selection policy} to select a leaf of the tree and repeats this procedure---solving the LP relaxation and branching on a variable. B\&C can \emph{fathom} a node, meaning that it will stop searching along that branch, if 1) the LP relaxation satisfies the IP's integrality constraints, 2) the LP relaxation is infeasible, or 3) the objective value of the LP relaxation's solution is no better than the best integral solution found thus far. We assume there is a bound $\kappa$ on the size of the tree we allow B\&C to build before we terminate, as is common in prior research~\citep{Hutter09:Paramils,Kleinberg17:Efficiency,Kleinberg19:Procrastinating,Balcan18:Learning}.

Cutting planes are a means of ensuring that at each iteration of B\&C, the solution to the LP relaxation is as close to the optimal integral solution as possible. Formally,
let $$\cP = \{\vec{x}\in\R^n : A\vec{x}\le\vec{b}, \vec{x}\ge 0\}$$ denote the feasible region obtained by taking the LP relaxation of IP~\eqref{eq:IP}. Let $\cP_I = \operatorname{conv}(\cP\cap\Z^n)$ denote the integer hull of $\cP$. A \emph{valid} cutting plane is any hyperplane $\vec{\alpha}^T\vec{x}\le\beta$ such that if $\vec{x}$ is in the integer hull $\left(\vec{x}\in\cP_I\right)$, then $\vec{x}$ satisfies the inequality $\vec{\alpha}^T\vec{x}\le\beta$. In other words, a valid cut does not remove any integral point from the LP relaxation's feasible region. A valid cutting plane \emph{separates} $\vec{x}\in\cP\setminus\cP_I$ if it does not satisfy the inequality, or in other words, $\vec{\alpha}^T\vec{x}>\beta$. At any node of the search tree, B\&C can add valid cutting planes that separate the optimal solution to the node's LP relaxation, thus improving the solution estimates used to prune the search tree. However, adding too many cuts will increase the time it takes to solve the LP relaxation at each node. Therefore, solvers such as SCIP~\citep{Gamrath20:SCIP}, the leading open-source solver, bound the number of cuts that will be applied.

A famous class of cutting planes is the family of \emph{Chv\'{a}tal-Gomory (CG) cuts}\footnote{The set of CG cuts is equivalent to the set of Gomory (fractional) cuts~\citep{Cornuejols01:Elementary}, another commonly studied family of cutting planes with a slightly different parameterization.}~\citep{Chvatal73:Edmonds,Gomory58:Outline}, which are parameterized by vectors $\vec{u}\in\R^m$. The CG cut defined by $\vec{u}\in\R^m$ is the hyperplane $$\lf\vec{u}^TA\rf\vec{x}\le\lf\vec{u}^T\vec{b}\rf,$$ which is guaranteed to be valid. Throughout this paper we primarily restrict our attention to $\vec{u}\in[0,1)^m$. This is without loss of generality, since the facets of $$\cP\cap\{\vec{x}\in\R^n: \lfloor\vec{u}^TA\rfloor\vec{x}\le\lfloor\vec{u}^T\vec{b}\rfloor\,\forall\vec{u}\in\R^m\}$$ can be described by the finitely many $\vec{u}\in[0,1)^m$ such that $\vec{u}^TA\in\Z^n$~\citep{Chvatal73:Edmonds}.

\begin{figure}
     \centering
     \begin{subfigure}[b]{0.18\textwidth}
        \captionsetup{justification=centering}
         \centering
         \includegraphics[width=.8\textwidth]{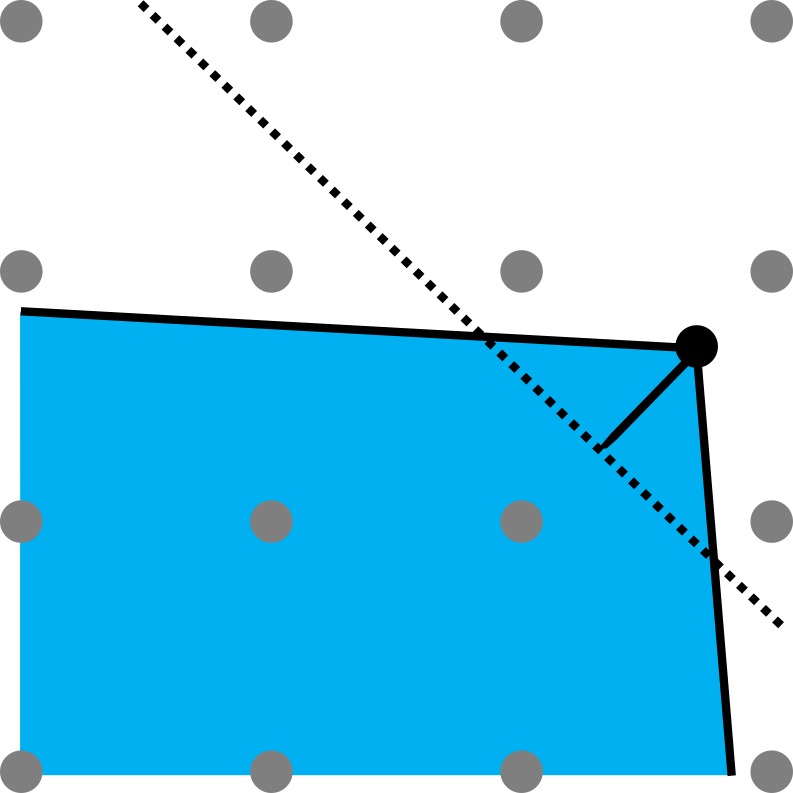}
         \caption{{\footnotesize Efficacy \newline}}
         \label{fig:efficacy}
     \end{subfigure}
     \hfill
     \begin{subfigure}[b]{0.18\textwidth}
        \captionsetup{justification=centering}
         \centering
         \includegraphics[width=.8\textwidth]{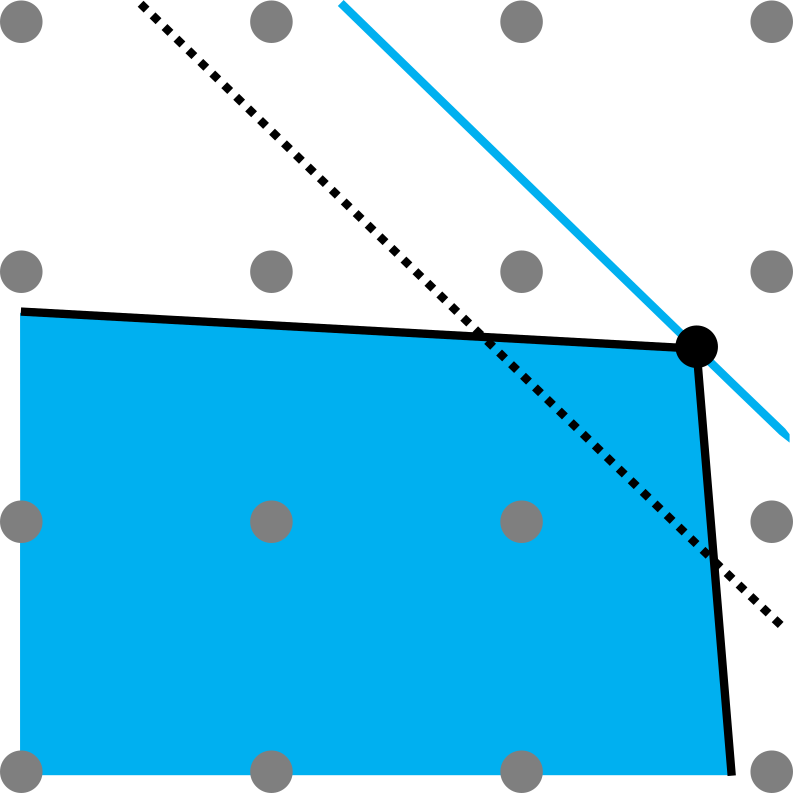}
         \caption{{\footnotesize Better objective parallelism}}
         \label{fig:parallel_good}
     \end{subfigure}
     \hfill
     \begin{subfigure}[b]{0.18\textwidth}
        \captionsetup{justification=centering}
         \centering
         \includegraphics[width=.8\textwidth]{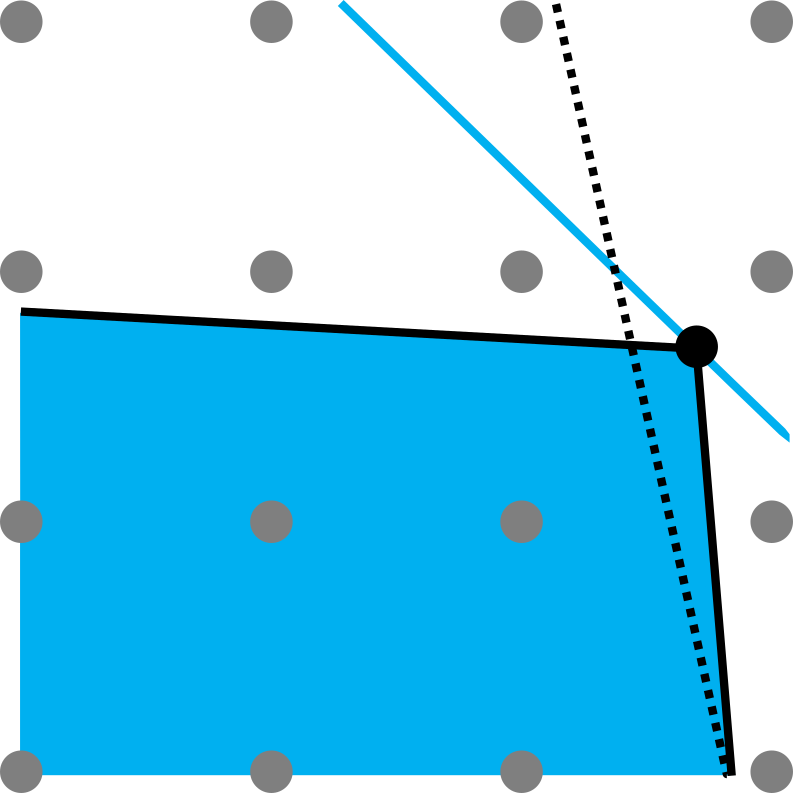}
         \caption{{\footnotesize Worse objective parallelism}}
         \label{fig:parallel_bad}
     \end{subfigure}
     \hfill
     \begin{subfigure}[b]{0.18\textwidth}
        \captionsetup{justification=centering}
         \centering
         \includegraphics[width=.8\textwidth]{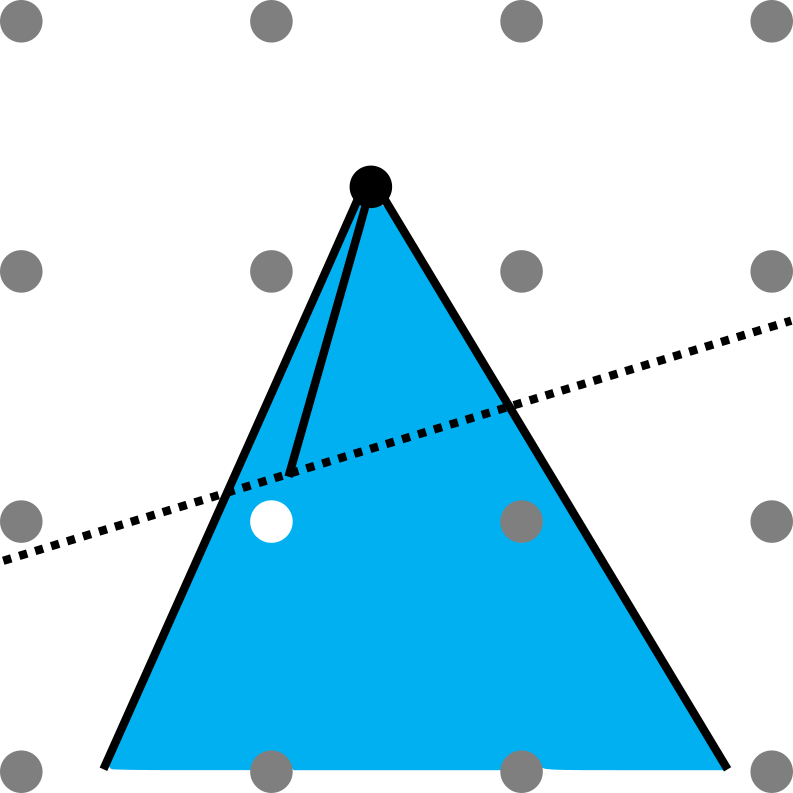}
         \caption{{\footnotesize Better directed cutoff distance}}
         \label{fig:directed_good}
     \end{subfigure}
     \hfill
     \begin{subfigure}[b]{0.18\textwidth}
     \captionsetup{justification=centering}
         \centering
         \includegraphics[width=.8\textwidth]{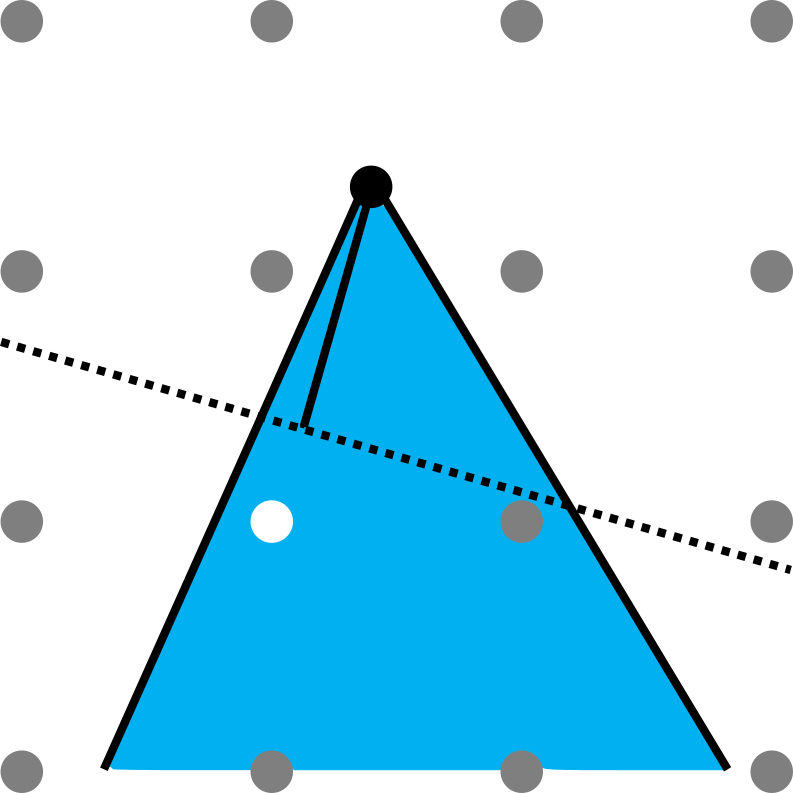}
         \caption{{\footnotesize Worse directed cutoff distance}}
         \label{fig:directed_bad}
     \end{subfigure}
        \caption{Illustration of scoring rules. In each figure, the blue region is the feasible region, the black dotted line is the cut in question, the blue solid line is orthogonal to the objective $\vec{c}$, the black dot is the LP optimal solution, and the white dot is the incumbent IP solution. Figure~\ref{fig:efficacy} illustrates efficacy, which is the length of the black solid line between the cut and the LP optimal solution. The cut in Figure~\ref{fig:parallel_good} has better objective parallelism than the cut in Figure~\ref{fig:parallel_bad}. The cut in Figure~\ref{fig:directed_good} has a better directed cutoff distance than the cut in Figure~\ref{fig:directed_bad}, but both have the same efficacy.}
        \label{fig:scoring_rules}
\end{figure}

Some IP solvers such as SCIP use \emph{scoring rules} to select among cutting planes, which are meant to measure the quality of a cut. Some commonly-used scoring rules include \emph{efficacy}~\citep{Balas96:Mixed} ($\score_1$), \emph{objective parallelism}~\citep{Achterberg07:Constraint} ($\score_2$), \emph{directed cutoff distance}~\citep{Gamrath20:SCIP} ($\score_3$), and \emph{integral support}~\citep{Wesselmann12:Implementing} ($\score_4$) (defined in Appendix~\ref{app:background}). \emph{Efficacy} measures the distance between the cut $\vec{\alpha}^T\vec{x}\le\beta$ and $\vec{x}^*_{\LP}$: $$\score_1(\vec{\alpha}^T\vec{x}\le\beta) = \frac{\vec{\alpha}^T\vec{x}^*_{\LP} - \beta}{\norm{\vec{\alpha}}_2},$$ as illustrated in Figure~\ref{fig:efficacy}. \emph{Objective parallelism} measures the angle between the objective $\vec{c}$ and the cut's normal vector $\vec{\alpha}$: $$\score_2(\vec{\alpha}^T\vec{x}\le\beta) = \frac{\left|\vec{c}^T\vec{\alpha}\right|}{\norm{\vec{\alpha}}_2 \norm{\vec{c}}_2},$$ as illustrated in Figures~\ref{fig:parallel_good} and \ref{fig:parallel_bad}. \emph{Directed cutoff distance} measures the distance between the LP optimal solution and the cut in a more relevant direction than the efficacy scoring rule. Specifically, let $\overline{\vec{x}}$ be the \emph{incumbent solution}, which is the best-known feasible solution to the input IP. The directed cutoff distance is the distance between the hyperplane $\left(\vec{\alpha}, \beta\right)$ and the current LP solution $\vec{x}^*_{\LP}$ along the direction of the incumbent $\overline{\vec{x}}$, as illustrated in Figures~\ref{fig:directed_good} and \ref{fig:directed_bad}: $$\score_3(\vec{\alpha}^T\vec{x}\le\beta) =  \frac{\vec{\alpha}^T \vec{x}^*_{\LP} - \beta}{\left|\vec{\alpha}^T \left(\overline{\vec{x}} - \vec{x}^*_{\LP}\right)\right|}\cdot\norm{\overline{\vec{x}} - \vec{x}^*_{\LP}}_2.$$  SCIP~\citep{Gamrath20:SCIP} uses the scoring rule $$\frac{3}{5}\score_1 + \frac{1}{10}\score_2 + \frac{1}{2}\score_3 + \frac{1}{10} \score_4.$$

\subsection{Learning theory background}\label{sec:LTbackground}

The goal of this paper is to learn cut-selection policies using samples in order to guarantee, with high probability, that B\&C builds a small tree in expectation on unseen IPs. To this end, we rely on the notion of \emph{pseudo-dimension}~\citep{Pollard84:Convergence}, a well-known measure of a function class's \emph{intrinsic complexity}. The pseudo-dimension of a function class $\cF\subseteq\R^{\cY}$, denoted $\pdim(\cF)$, is the largest integer $N$ for which there exist $N$ inputs $y_1,\ldots, y_N \in \cY$ and $N$ thresholds $r_1,\ldots, r_N \in \R$ such that for every $(\sigma_1,\ldots,\sigma_N)\in\{0,1\}^N$, there exists $f\in\cF$ such that $f(y_i)\ge r_i$ if and only if $\sigma_i=1$. Function classes with bounded pseudo-dimension satisfy the following uniform convergence guarantee~\cite{Anthony09:Neural,Pollard84:Convergence}.
Let $[-\kappa, \kappa]$ be the range of the functions in $\cF$, let $$N_{\cF}(\varepsilon,\delta) = O\left(\frac{\kappa^2}{\varepsilon^2}\left(\pdim(\cF) + \ln\left(\frac{1}{\delta}\right)\right)\right),$$ and let $N\ge N_{\cF}(\varepsilon,\delta)$. For all distributions $\dist$ on $\cY$, with probability $1-\delta$ over the draw of $y_1,\ldots, y_N\sim\dist$, for every function $f \in \cF$, the average value of $f$ over the samples is within $\varepsilon$ of its expected value: $$\left|\frac{1}{N}\sum_{i=1}^Nf(y_i)-\E_{y\sim\dist}[f(y)]\right|\le\varepsilon.$$ The quantity $N_{\cF}(\varepsilon,\delta)$ is the \emph{sample complexity} of $\cF$.

\section{Learning Chv\'{a}tal-Gomory cuts}\label{sec:cg_cuts}

In this section we bound the sample complexity of learning CG cuts at the root node of the B\&C search tree. We warm up by analyzing the case where a single CG cut is added at the root (Section~\ref{sec:single}), and then build on this analysis to handle $W$ sequential \emph{waves} of $k$ simultaneous CG cuts (Section~\ref{sec:waves_simultaneous}). This means that all $k$ cuts in the first wave are added simultaneously, the new (larger) LP relaxation is solved, all $k$ cuts in the second wave are added to the new problem simultaneously, and so on. B\&C adds cuts in waves because otherwise the angles between cuts would become obtuse, leading to numerical instability. Moreover, many commercial IP solvers only add cuts at the root because those cuts can be leveraged throughout the tree. However, in Section~\ref{sec:tree_search}, we also provide guarantees for applying cuts throughout the tree. In this section, we assume that all aspects of B\&C (such as node selection and variable selection) are fixed except for the cuts applied at the root of the search tree. 

\subsection{Learning a single cut}\label{sec:single}
To provide sample complexity bounds, as per Section~\ref{sec:LTbackground},
we bound the pseudo-dimension of the set of functions $f_{\vec{u}}$ for $\vec{u} \in [0,1]^m$, where $f_{\vec{u}}(\vec{c}, A, \vec{b})$ is the size of the tree B\&C builds when it applies the CG cut defined by $\vec{u}$ at the root.
To do so, we take advantage of structure exhibited by the class of \emph{dual} functions, each of which is defined by a fixed IP $(\vec{c}, A, \vec{b})$ and measures tree size as a function of the parameters $\vec{u}$. In other words, each dual function $f_{\vec{c}, A, \vec{b}}^*: [0,1]^m \to \R$ is defined as $f_{\vec{c}, A, \vec{b}}^*(\vec{u}) = f_{\vec{u}}(\vec{c}, A, \vec{b})$. Our main result in this section is a proof that the dual functions are well-structured (Lemma~\ref{lem:1plane}), which then allows us to apply a result by \citet{Balcan21:How} to bound $\pdim(\{f_{\vec{u}} : \vec{u} \in [0,1]^m\})$ (Theorem~\ref{thm:1plane}). Proving that the dual functions are well-structured is challenging because they are volatile: slightly perturbing $\vec{u}$ can cause the tree size to shift from constant to exponential in $n$, as we prove in the following theorem. The full proof is in Appendix~\ref{app:cg_cuts}.
\begin{theorem}\label{thm:sensitive}
    For any integer $n$, there exists an integer program $(\vec{c}, A, \vec{b})$ with two constraints and $n$ variables such that if $\frac{1}{2} \leq u[1] - u[2] < \frac{n+1}{2n}$, then applying the CG cut defined by $\vec{u}$ at the root causes B\&C to terminate immediately. Meanwhile, if $\frac{n+1}{2n} \leq u[1] - u[2] < 1$, then applying the CG cut defined by $\vec{u}$ at the root causes B\&C to build a tree of size at least $2^{(n-1)/2}.$
\end{theorem}

\begin{proof}[Proof sketch]
Without loss of generality, assume that $n$ is odd. Consider an IP with constraints $2(x[1]+\cdots+x[n])\le n$, $-2(x[1]+\cdots+x[n])\le -n$, $\vec{x}\in\{0,1\}^n$, and any objective. This IP is infeasible because $n$ is odd. \citet{Jeroslow74:Trivial} proved that without the use of cutting planes or heuristics, B\&C will build a tree of size $2^{(n-1)/2}$ before it terminates. We prove that when $\frac{1}{2}\le u[1]-u[2]<\frac{n+1}{2n}$, the CG cut halfspace defined by $\vec{u} = (u[1], u[2])$ has an empty intersection with the feasible region of the IP, causing B\&C to terminate immediately. This is illustrated in Figure~\ref{fig:good_cut}.
\begin{figure}
     \centering
     \begin{subfigure}[b]{0.47\textwidth}
         \centering
         \includegraphics{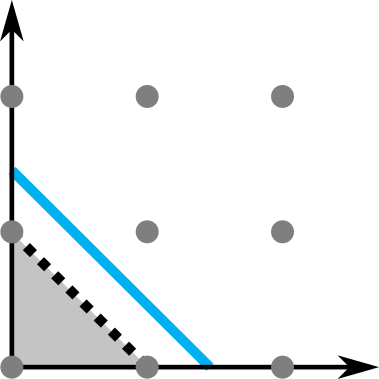}
         \caption{Cut produced when $\frac{1}{2} \leq u[1] - u[2] < \frac{2}{3}$. The grey solid region is the set of points $\vec{x}$ such that $x[1] + x[2] \leq 1.$}
         \label{fig:good_cut}
     \end{subfigure}
     \qquad
     \begin{subfigure}[b]{0.47\textwidth}
         \centering
         \includegraphics{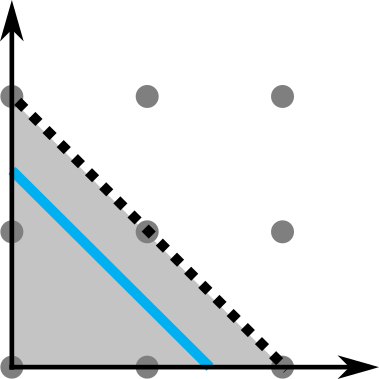}
         \caption{Cut produced when $\frac{2}{3} \leq u[1] - u[2] < 1$. The grey solid region is the set of points $\vec{x}$ such that $x[1] + x[2] \leq 2.$}
         \label{fig:bad_cut}
     \end{subfigure}
        \caption{Illustration of Theorem~\ref{thm:sensitive} when $n = 3$, projected onto the $x[3] = 0$ plane. The blue solid line is the feasible region $2x[1] + 2x[2] = 3$. The black dotted lines are the cut.}\label{fig:sensitivity}
\end{figure}
On the other hand, we show that if $\frac{n+1}{2n}\le u[1]-u[2]<1$, then the CG cut halfspace defined by $\vec{u}$ contains the feasible region of the IP, and thus leaves the feasible region unchanged. This is illustrated by Figure~\ref{fig:bad_cut}. In this case, due to~\citet{Jeroslow74:Trivial}, applying this CG cut at the root will cause B\&C to build a tree of size at least $2^{(n-1)/2}$ before it terminates.
\end{proof}

This theorem shows that the dual tree-size functions can be extremely sensitive to perturbations in the CG cut parameters. However, we are able to prove that the dual functions are piecewise-constant.

\begin{lemma}\label{lem:1plane}
For any IP $(\vec{c}, A, \vec{b})$, there are $O(\norm{A}_{1,1} + \norm{\vec{b}}_{1,1} + n)$ hyperplanes that partition $[0,1]^m$ into regions where in any one region $R$, the dual function $f_{\vec{c}, A, \vec{b}}^*(\vec{u})$ is constant for all $\vec{u} \in R$.
\end{lemma}

\begin{proof}
Let $\vec{a}_1, \dots, \vec{a}_n \in \R^m$ be the columns of $A$. Let $A_i =\norm{\vec{a}_i}_1$ and $B = \norm{\vec{b}}_1$, so for any $\vec{u}\in [0, 1]^m$, $\lf \vec{u}^T \vec{a}_i\rf\in [-A_i, A_i]$ and $\lf \vec{u}^T \vec{b}\rf\in [-B, B]$. For each integer $k_i\in [-A_i, A_i]$, we have $$\lf \vec{u}^T \vec{a}_i\rf = k_i \iff k_i\le \vec{u}^T \vec{a}_i < k_i+1.$$ There are $\sum_{i = 1}^n 2A_i + 1 = O(\norm{A}_{1,1} + n)$ such halfspaces, plus an additional $2B+1$ halfspaces of the form $k_{n+1}\le \vec{u}^T \vec{b} < k_{n+1}+1$ for each $k_{n+1}\in\{-B,\ldots, B\}$. In any region $R$ defined by the intersection of these halfspaces, the vector $(\lfloor \vec{u}^T \vec{a}_1\rfloor,\ldots, \lfloor \vec{u}^T\vec{a}_n\rfloor, \lfloor \vec{u}^T\vec{b}\rfloor)$ is constant for all $\vec{u} \in R$, and thus so is the resulting cut.
\end{proof}

Combined with the main result of~\citet{Balcan21:How}, this lemma implies the following bound.

\begin{theorem}\label{thm:1plane}
Let $\cF_{\alpha, \beta}$ denote the set of all functions $f_{\vec{u}}$ for $\vec{u}\in[0,1]^m$ defined on the domain of IPs $(\vec{c}, A, \vec{b})$ with $\norm{A}_{1,1}\le\alpha$ and $\norm{\vec{b}}_1\le\beta$. Then, $\pdim(\cF_{\alpha, \beta})=O(m \log (m(\alpha + \beta + n))).$
\end{theorem}

This theorem implies that $\widetilde O(\kappa^2 m / \varepsilon^2)$ samples are sufficient to ensure that with high probability, for every CG cut, the average size of the tree B\&C builds upon applying the cutting plane is within $\epsilon$ of the expected size of the tree it builds (the $\widetilde{O}$ notation suppresses logarithmic terms).

\subsection{Learning a sequence of cuts}

We now determine the sample complexity of making $W$ sequential CG cuts at the root. The first cut is defined by $m$ parameters $\vec{u}_{1}\in [0, 1]^m$ for each of the $m$ constraints. Its application leads to the addition of the row $\lfloor \vec{u}_1^TA \rfloor \vec{x} \leq \lfloor \vec{u}_1^T \vec{b}\rfloor$ to the constraint matrix. The next cut is then defined by $m+1$ parameters $\vec{u}_2 \in [0,1]^{m+1}$ since there are now $m+1$ constraints. Continuing in this fashion, the $W$th cut is defined by $m+W-1$ parameters $\vec{u}_{W}\in [0, 1]^{m+W-1}$. Let $f_{\vec{u}_{1},\ldots,\vec{u}_{W}}(\vec{c}, A, \vec{b})$ be the size of the tree B\&C builds when it applies the CG cut defined by $\vec{u}_{1}$, then applies the CG cut defined by $\vec{u}_{2}$ to the new IP, and so on, all at the root of the search tree.

As in Section~\ref{sec:single}, we bound the pseudo-dimension of the functions $f_{\vec{u}_{1},\ldots,\vec{u}_{W}}$ by analyzing the structure of the dual functions $f_{\vec{c}, A, \vec{b}}^*$, which measure tree size as a function of the parameters $\vec{u}_{1},\ldots,\vec{u}_{W}$. Specifically, $f_{\vec{c}, A, \vec{b}}^*: [0,1]^m\times\cdots\times [0, 1]^{m+W-1} \to \R$, where $f_{\vec{c}, A, \vec{b}}^*(\vec{u}_{1},\ldots,\vec{u}_{W}) = f_{\vec{u}_{1},\ldots,\vec{u}_{W}}(\vec{c}, A, \vec{b})$. The analysis in this section is more complex because the $w^{th}$ cut (with $w \in \{2, \dots, W\}$) depends not only on the parameters $\vec{u}_w$ but also on $\vec{u}_1, \dots, \vec{u}_{w-1}$. We prove that the dual functions are again piecewise-constant, but in this case, the boundaries between pieces are hypersurfaces defined by multivariate polynomials rather than hyperplanes. The full proof is in Appendix~\ref{app:cg_cuts}.

\begin{lemma}\label{lem:Wplane}
For any IP $(\vec{c}, A, \vec{b})$, there are $O(W2^W\norm{A}_{1,1} + 2^W\norm{\vec{b}}_1 + nW)$ multivariate polynomials in $\le W^2+mW$ variables of degree $\le W$ that partition $[0,1]^m\times\cdots\times [0, 1]^{m+W-1}$ into regions where in any one region $R$, $f_{\vec{c}, A, \vec{b}}^*(\vec{u}_{1},\ldots,\vec{u}_{W})$ is constant for all $(\vec{u}_{1},\ldots, \vec{u}_{W}) \in R$.
\end{lemma}

\begin{proof}[Proof sketch]
Let $\vec{a}_1, \dots, \vec{a}_n \in \R^m$ be the columns of $A$. For $\vec{u}_1\in [0,1]^m,\ldots, \vec{u}_W\in [0,1]^{m+W-1}$, define $\vec{\widetilde{a}}^1_i\in [0,1]^m,\ldots, \vec{\widetilde{a}}^W_i\in[0,1]^{m+W-1}$ for each $i \in [n]$ such that $\vec{\widetilde{a}}_i^w$ is the $i$th column of the constraint matrix after applying cuts $\vec{u}_1,\ldots,\vec{u}_{w-1}$. Similarly, define $\vec{\widetilde{b}}^w$ to be the constraint vector after applying the first $w-1$ cuts. In other words, we have the recurrence relation
\begin{alignat*}{2}
&\vec{\widetilde{a}}^1_i = \vec{a}_i &&\vec{\widetilde{b}}^1 = \vec{b} \\
    &\vec{\widetilde{a}}^w_i = \begin{bmatrix}\vec{\widetilde{a}}^{w-1}_i \\ \vec{u}^T_{w-1}\vec{\widetilde{a}}^{w-1}_i \end{bmatrix} \qquad     &&\vec{\widetilde{b}}^w = \begin{bmatrix}\vec{\widetilde{b}}^{w-1} \\ \vec{u}_{w-1}^T\vec{\widetilde{b}}^{w-1}\end{bmatrix}
\end{alignat*} 
for $w = 2,\ldots, W$. We prove, by induction, that $\lf\vec{u}_w^T\widetilde{\vec{a}}_i^w\rf \in \left[-2^{w-1}\norm{\vec{a}_i}_1, 2^{w-1}\norm{\vec{a}_i}_1\right].$ For each integer $k_i$ in this interval, $$\lf\vec{u}_w^T\vec{\widetilde{a}}_i^w\rf = k_i \iff k_i\le \vec{u}_w^T\vec{\widetilde{a}}_i^w < k_i+1.$$ The boundaries of these surfaces are defined by polynomials over $\vec{u}_w$ in $\le mw+w^2$ variables with degree $\le w$. Counting the total number of such hypersurfaces yields the lemma statement.
\end{proof}

We now use this structure to provide a pseudo-dimension bound. The full proof is in Appendix~\ref{app:cg_cuts}.

\begin{theorem}\label{thm:Wpdim}
Let $\cF_{\alpha, \beta}$ denote the set of all functions $f_{\vec{u}_1,\ldots,\vec{u}_W}$ for $\vec{u}_1\in [0,1]^m,\ldots,\vec{u}_W\in[0,1]^{m+W-1}$ defined on the domain of integer programs $(\vec{c}, A, \vec{b})$ with $\norm{A}_{1,1}\le\alpha$ and $\norm{\vec{b}}_1\le\beta$. Then, $\pdim(\cF_{\alpha, \beta}) =O(mW^2\log(mW(\alpha+\beta+n))).$
\end{theorem}

\begin{proof}[Proof sketch]
The space of $0/1$ classifiers induced by the set of degree $\le W$ multivariate polynomials in $W^2+mW$ variables has VC dimension $O((W^2+mW)\log W)$~\cite{Anthony09:Neural}. However, we more carefully examine the structure of the polynomials considered in Lemma~\ref{lem:Wplane} to give an improved VC dimension bound of $1+mW$. For each $j=1,\ldots,m$ define $\vec{\widetilde{u}}_1[j],\ldots, \vec{\widetilde{u}}_W[j]$ recursively as
\begin{align*}
    \vec{\widetilde{u}}_1[j] &= \vec{u}_1[j] \\
    \vec{\widetilde{u}}_w[j] &= \vec{u}_w[j] + \sum_{\ell = 1}^{w-1}\vec{u}_w[m+\ell]\vec{\widetilde{u}}_{\ell}[j]\qquad\text{for }w = 2,\ldots, W.
\end{align*} The space of polynomials induced by the $w$th cut is contained in $\operatorname{span}\{1, \vec{\widetilde{u}}_w[1],\ldots,\vec{\widetilde{u}}_w[m]\}$. The intuition for this is as follows: consider the additional term added by the $w$th cut to the constraint matrix, that is, $\vec{u}_w^T\widetilde{\vec{a}}_i^w$. The first $m$ coordinates $(\vec{u}_w[1],\ldots,\vec{u}_w[m])$ interact only with $\vec{a}_i$---so $\vec{u}_w[j]$ collects a coefficient of $\vec{a}_i[j]$. Each subsequent coordinate $\vec{u}_w[m+\ell]$ interacts with all coordinates of $\widetilde{\vec{a}}_i^w$ arising from the first $\ell$ cuts. The term that collects a coefficient of $\vec{a}_i[j]$ is precisely $\vec{u}_w[m+\ell]$ times the sum of all terms from the first $\ell$ cuts with a coefficient of $\vec{a}_i[j]$.
Using standard facts about the VC dimension of vector spaces and their duals in conjunction with Lemma~\ref{lem:Wplane} and the framework of~\citet{Balcan21:How} yields the theorem statement.
\end{proof}

The sample complexity of learning $W$ sequential cuts is thus $\widetilde O(\kappa^2mW^2 / \epsilon^2).$

\subsection{Learning waves of simultaneous cuts}\label{sec:waves_simultaneous}

We now determine the sample complexity of making $W$ sequential waves of cuts at the root, each wave consisting of $k$ simultaneous CG cuts. Given vectors $\vec{u}^1_{1},\ldots,\vec{u}^k_{1}\in [0, 1]^m, \vec{u}^1_{2},\ldots,\vec{u}^k_2\in [0, 1]^{m+k},\ldots, \vec{u}^1_{W},\ldots,\vec{u}^k_W\in [0, 1]^{m+k(W-1)}$, let $f_{\vec{u}^1_{1},\ldots,\vec{u}^k_{1},\ldots,\vec{u}^1_{W},\ldots,\vec{u}^k_W}(\vec{c}, A, \vec{b})$ be the size of the tree B\&C builds when it applies the CG cuts defined by $\vec{u}^1_{1},\ldots,\vec{u}^k_1$, then applies the CG cuts defined by $\vec{u}^1_{2},\ldots,\vec{u}^k_2$ to the new IP, and so on, all at the root of the search tree. The full proof of the following theorem is in Appendix~\ref{app:cg_cuts}, and follows from the observation that $W$ waves of $k$ simultaneous cuts can be viewed as making $kW$ sequential cuts with the restriction that cuts within each wave assign nonzero weight only to constraints from previous waves.

\begin{theorem}\label{thm:Wkcuts_pdim}
Let $\cF_{\alpha, \beta}$ be the set of all functions $f_{\vec{u}^1_1,\ldots,\vec{u}^k_1,\ldots,\vec{u}^1_W,\ldots,\vec{u}^k_W}$ for $\vec{u}^1_{1},\ldots,\vec{u}^k_{1}\in [0, 1]^m,\ldots$, $\vec{u}^1_{W},\ldots,\vec{u}^k_W\in [0, 1]^{m+k(W-1)}$ defined on the domain of integer programs $(\vec{c}, A, \vec{b})$ with $\norm{A}_{1,1}\le\alpha$ and $\norm{\vec{b}}_1\le\beta$. Then, $\pdim(\cF_{\alpha, \beta})=O(mk^2W^2\log(mkW(\alpha+\beta+n))).$
\end{theorem}

This result implies that the sample complexity of learning $W$ waves of $k$ cuts is $\widetilde O(\kappa^2mk^2W^2 / \epsilon^2).$

\subsection{Data-dependent guarantees}
So far, our guarantees have depended on the maximum possible norms of the constraint matrix and vector in the domain of IPs under consideration. The uniform convergence result in Section~\ref{sec:LTbackground} for $\cF_{\alpha,\beta}$ only holds for distributions over $A$ and $\vec{b}$ with norms bounded by $\alpha$ and $\beta$, respectively. In Appendix~\ref{app:data}, we show how to convert these into more broadly applicable data-dependent guarantees that leverage properties of the distribution over IPs. These guarantees hold without assumptions on the distribution's support, and depend on $\E[\max_i\norm{A_i}_{1,1}]$ and $\E[\max_i\norm{\vec{b}_i}_1]$ (where the expectation is over the draw of $N$ samples), thus giving a sharper sample complexity guarantee that is tuned to the distribution.

\section{Learning cut selection policies}\label{sec:cut_policies}

In Section~\ref{sec:cg_cuts}, we studied the sample complexity of learning waves of specific cut parameters. In this section, we bound the sample complexity of learning \emph{cut-selection policies} at the root, that is, functions that take as input an IP and output a candidate cut. This is a more nuanced way of choosing cuts since it allows for the cut parameters to depend on the input IP. 

Formally, let $\cI_{m}$ be the set of IPs with $m$ constraints (the number of variables is always fixed at $n$) and let $\cH_{m}$ be the set of all hyperplanes in $\R^m$. A \emph{scoring rule} is a function $\score:\cup_m(\cH_{m}\times\cI_m)\to\R_{\ge 0}$. The real value $\score(\vec{\alpha}^T\vec{x}\le\beta, (\vec{c}, A, \vec{b}))$ is a measure of the quality of the cutting plane $\vec{\alpha}^T\vec{x}\le\beta$ for the IP $(\vec{c}, A, \vec{b})$. Examples include the scoring rules discussed in Section~\ref{sec:BandC}. Given a scoring rule and a family of cuts, a cut-selection policy applies the cut from the family with maximum score.

Suppose $\score_1,\ldots,\score_d$ are $d$ different scoring rules. We bound the sample complexity of learning a combination of these scoring rules that guarantees a low expected tree size. 

\begin{theorem}\label{thm:scoring}
	Let $\cC$ be a set of cutting-plane parameters such that for every IP $(\vec{c}, A, \vec{b})$, there is a decomposition of $\cC$ into $\le r$ regions such that the cuts generated by any two vectors in the same region are the same. Let $\score_1,\ldots,\score_d$ be $d$ scoring rules. For $\vec{\mu}\in\R^d$, let $f_{\vec{\mu}}(\vec{c}, A, \vec{b})$ be the size of the tree B\&C builds when it chooses a cut from $\cC$ to maximize $\mu[1]\score_1(\cdot, (\vec{c}, A, \vec{b}))+\cdots+\mu[d]\score_d(\cdot, (\vec{c}, A, \vec{b})).$ Then, $\pdim(\{f_{\vec{\mu}} : \vec{\mu} \in \R^d\})=O(d \log(rd))$.
\end{theorem}

\begin{proof}
Fix an integer program $(\vec{c}, A, \vec{b})$. Let $\vec{u}_1,\ldots,\vec{u}_r\in \cC$ be arbitrary cut parameters from each of the $r$ regions. Consider the hyperplanes $$\sum_{i = 1}^d \mu[i]\score_i(\vec{u}_s) = \sum_{i=1}^d\mu[i]\score_i(\vec{u}_t)$$ for each $s\neq t\in\{1,\ldots, r\}$ (suppressing the dependence on $\vec{c}, A, \vec{b}$). These $O(r^2)$ hyperplanes partition $\R^d$ into regions such that as $\vec{\mu}$ varies in a given region, the cut chosen from $\cC$ is invariant. The desired pseudo-dimension bound follows from the main result of \citet{Balcan21:How}. \end{proof}

Theorem~\ref{thm:scoring} can be directly instantiated with the class of CG cuts. Combining Lemma~\ref{lem:1plane} with the basic combinatorial fact that $k$ hyperplanes partition $\R^{m}$ into at most $k^m$ regions, we get that the pseudo-dimension of $\{f_{\vec{\mu}}:\vec{\mu}\in\R^d\}$ defined on IPs with $\norm{A}_{1,1}\le\alpha$ and $\norm{\vec{b}}_1\le\beta$ is $O(dm\log(d(\alpha+\beta+n)))$. Instantiating Theorem~\ref{thm:scoring} with the set of all sequences of $W$ CG cuts requires the following extension of scoring rules to sequences of cutting planes. A \emph{sequential scoring rule} is a function that takes as input an IP $(\vec{c}, A, \vec{b})$ and a sequence of cutting planes $h_1,\ldots, h_W$, where each cut lives in one higher dimension than the previous. It measures the quality of this sequence of cutting planes when applied one after the other to the original IP. Every scoring rule $\score$ can be naturally extended to a sequential scoring rule $\overline{\score}$ defined by $\overline{\score}(h_1,\ldots, h_W, (\vec{c}^0, A^0, \vec{b}^0)) = \sum_{i=0}^{d-1}\score(h_{i+1}, (\vec{c}^{i}, A^{i}, \vec{b}^{i}))$, where $(\vec{c}^i, A^i, \vec{b}^i)$ is the IP after applying cuts $h_1,\ldots, h_{i-1}$.

\begin{corollary} Let $\cC = [0,1]^m\times\cdots\times [0,1]^{m+W-1}$ denote the set of possible sequences of $W$ Chv\'{a}tal-Gomory cut parameters. Let $\score_1,\ldots,\score_d : \cC\times\cI_m\times\cdots\times\cI_{m+W-1}\to\R$ be $d$ sequential scoring rules and let $f_{\vec{\mu}}(\vec{c}, A, \vec{b})$ be as in Theorem~\ref{thm:scoring} for the class $\cC$. Then, $\pdim(\{f^W_{\vec{\mu}}:\vec{\mu}\in\R^d\})=O(dmW^2\log(dW(\alpha + \beta +n)))$.
\end{corollary}

\begin{proof}
In Lemma~\ref{lem:Wplane} and Theorem~\ref{thm:Wpdim} we showed that there are $O(W2^W\alpha+2^W\beta+nW)$ multivariate polynomials that belong to a family of polynomials $\cG$ with $\VC(\cG^*)\le 1+mW$ ($\cG^*$ denotes the dual of $\cG$) that partition $\cC$ into regions such that resulting sequence of cuts is invariant in each region. By Claim 3.5 by \citet{Balcan21:How}, the number of regions is $$O(W2^W\alpha+2^W\beta+nW)^{\VC(\cG^*)} \le O(W2^W\alpha+2^W\beta+nW)^{1+mW}.$$ The corollary then follows from Theorem~\ref{thm:scoring}.
\end{proof}

These results bound the sample complexity of learning cut-selection policies based on scoring rules, which allow the cuts that B\&C selects to depend on the input IP.

\section{Sample complexity of generic tree search}\label{sec:tree_search}

In this section, we study the sample complexity of selecting high-performing parameters for generic tree-based algorithms, which are a generalization of B\&C. This abstraction allows us to provide guarantees for simultaneously optimizing key aspects of tree search beyond cut selection, including node selection and branching variable selection. We also generalize the previous sections by allowing actions (such as cut selection) to be taken at any stage of the tree search---not just at the root.

Tree search algorithms take place over a series of $\kappa$ \emph{rounds} (analogous to the B\&C tree-size cap $\kappa$ in the previous sections). There is a sequence of $t$ \emph{steps} that the algorithm takes on each round. For example, in B\&C, these steps include node selection, cut selection, and variable selection. The specific \emph{action} the algorithm takes during each step (for example, which node to select, which cut to include, or which variable to branch on) typically depends on a \emph{scoring rule}. This scoring rule weights each possible action and the algorithm performs the action with the highest weight.
These actions (deterministically) transition the algorithm from one \emph{state} to another. This high-level description of tree search is summarized by Algorithm~\ref{alg:TS}.
For each step $j \in [t]$, the number of possible actions is $T_j \in \N$. There is a scoring rule $\score_j$, where $\score_j(k,s) \in \R$ is the weight associated with the action $k \in \left[T_j\right]$ when the algorithm is in the state $s$.
\begin{algorithm}
	\caption{Tree search}\label{alg:TS}
	\begin{algorithmic}[1]
		\Require Problem instance, $t$ scoring rules $\score_1, \dots, \score_t$, number of rounds $\kappa$.
		\State $s_{1,1} \leftarrow$ Initial state of algorithm
		\For {each round $i \in [\kappa]$}
		\For {each step $j \in [t]$}
		\State Perform the action $k \in \left[T_j\right]$ that maximizes $\score_j\left(k, s_{i,j}\right)$\label{step:score}
		\State $s_{i, j + 1}\leftarrow$ New state of algorithm
		\EndFor
		\State $s_{i+1, 1} \leftarrow s_{i, t + 1}$ \Comment{State at beginning of next round equals state at end of this round}
		\EndFor
		\Ensure Incumbent solution in state $s_{\kappa, t+1}$, if one exists.
	\end{algorithmic}
\end{algorithm}

There are often several scoring rules one could use, and it is not clear which to use in which scenarios. As in Section~\ref{sec:cut_policies}, we provide guarantees for learning combinations of these scoring rules for the particular application at hand. More formally, for each step $j \in [t]$, rather than just a single scoring rule $\score_j$ as in Step~\ref{step:score}, there are $d_j$ scoring rules $\score_{j,1}, \dots, \score_{j,d_j}$. Given parameters $\vec{\mu}_j = \left(\mu_{j}[1], \dots, \mu_{j}[d_j]\right) \in \R^{d_j}$, the algorithm takes the action $k \in [T_j]$ that maximizes $\sum_{i = 1}^{d_j}\mu_j[i]\score_{j,i}(k,s)$. There is a distribution $\dist$ over inputs $x$ to Algorithm~\ref{alg:TS}. For example, when this framework is instantiated for B\&C, $x$ is an integer program $(\vec{c}, A, \vec{b})$. There is a utility function $f_{\vec{\mu}}(x) \in [-H, H]$ that measures the utility of the algorithm parameterized by $\vec{\mu}= \left(\vec{\mu}_1, \dots, \vec{\mu}_t\right)$ on input $x$. For example, this utility function might measure the size of the search tree that the algorithm builds (in which case one can take $H\le\kappa$). We assume that this utility function is \emph{final-state-constant}:
\begin{definition}
Let $\vec{\mu}= \left(\vec{\mu}_1, \dots, \vec{\mu}_t\right)$ and $\vec{\mu}'= \left(\vec{\mu}_1', \dots, \vec{\mu}_t'\right)$ be two parameter vectors. Suppose that we run Algorithm~\ref{alg:TS} on input $x$ once using the scoring rule $\score_j = \sum_{i = 1}^{d_j}\mu_{j}[i]\score_{j,i}$ and once using the scoring rule $\score_j = \sum_{i = 1}^{d_j}\mu_{j}'[i]\score_{j,i}$. Suppose that on each run, we obtain the same final state $s_{\kappa, t+1}$. The utility function is \emph{final-state-constant} if $f_{\vec{\mu}}(x) = f_{\vec{\mu}'}(x)$.
\end{definition}
We provide a sample complexity bound for learning the parameters $\vec{\mu}$. The full proof is in Appendix~\ref{app:tree_search}.
 
\begin{theorem}\label{thm:main}
	Let $d  = \sum_{j = 1}^t d_j$ denote the total number of tunable parameters of tree search. Then, $$\pdim (\{f_{\vec{\mu}} : \vec{\mu} \in \R^d\})=O\Bigg(d \kappa \sum_{j = 1}^t \log T_j + d\log d \Bigg).$$
\end{theorem}

\begin{proof}[Proof sketch]
We prove that there is a set of hyperplanes splitting the parameter space into regions such that if tree search uses any parameter setting from a single region, it will always take the same sequence of actions (including node, variable, and cut selection). The main subtlety is an induction argument to count these hyperplanes that depends on the current step of the tree-search algorithm.
\end{proof}

In the context of integer programming, Theorem~\ref{thm:main} not only recovers the main result of~\citet{Balcan18:Learning} for learning variable selection policies, but also yields a more general bound that simultaneously incorporates cutting plane selection, variable selection, and node selection. In B\&C, the first action of each round is to select a node. Since there are at most $\kappa$ nodes expanded by B\&C, $T_1 \le \kappa$. The second action is to choose a cutting plane. As in Theorem~\ref{thm:scoring}, let $\cC$ be a family of cutting planes such that for every IP $(\vec{c}, A, \vec{b})$, there is a decomposition of the parameter space into $\le r$ regions such that the cuts generated by any two parameters in the same region are the same. Therefore, $T_2\le r$. The last action is to choose a variable to branch on at that node, so $T_3 = n$. Applying Theorem~\ref{thm:main}, $$\pdim(\{f_{\vec{\mu}} : \vec{\mu} \in \R^d\})=O\left(d \kappa(\log\kappa + \log r +\log n) + d\log d \right).$$ Ignoring $T_1$ and $T_2$, thereby only learning the variable selection policy, recovers the $O(d\kappa\log n + d\log d)$ bound of~\citet{Balcan18:Learning}.

\section{Conclusions and future research}
We provided the first provable guarantees for using machine learning to configure cutting planes and cut-selection policies. We analyzed the sample complexity of learning cutting planes from the popular family of Chv\'{a}tal-Gomory (CG) cuts. We then provided sample complexity guarantees for learning parameterized cut-selection policies, which allow the branch-and-cut algorithm to adaptively apply cuts as it builds the search tree. We showed that this analysis can be generalized to simultaneously capture various key aspects of tree search beyond cut selection, such as node and variable selection.

This paper opens up a variety questions for future research. For example, which other cut families can we learn over with low sample complexity? Section~\ref{sec:cg_cuts} focused on learning within the family of CG cuts (Sections~\ref{sec:cut_policies} and \ref{sec:tree_search} applied more generally). There are many other families, such as \emph{Gomory mixed-integer cuts} and \emph{lift-and-project cuts}, and a sample complexity analysis of these is an interesting direction for future research (and would call for new techniques). In addition, can we use machine learning to design improved scoring rules and heuristics for cut selection?

\subsection*{Acknowledgements}
This material is based on work supported by the National Science Foundation under grants IIS-1718457, IIS-1901403, IIS-1618714, and CCF-1733556, CCF-1535967, CCF-1910321, and SES-1919453, the ARO under award W911NF2010081, the Defense Advanced Research Projects Agency under cooperative agreement HR00112020003, an AWS Machine Learning Research Award, an Amazon Research Award, a Bloomberg Research Grant, and a Microsoft Research Faculty Fellowship.

\bibliographystyle{plainnat}
\bibliography{dairefs}


\appendix

\section{Additional background about cutting planes}\label{app:background}

\paragraph{Integral support~\citep{Wesselmann12:Implementing}.}  Let $Z$ be the set of all indices $\ell \in [n]$ such that $\vec{\alpha}[\ell] \not= 0$. Let $\bar{Z}$ be the set of all indices $\ell \in Z$ such that the $\ell^{th}$ variable is constrained to be integral. This scoring rule is defined as \[\score_4(\vec{\alpha}^T\vec{x}\le\beta) = \frac{\left|\bar{Z}\right|}{|Z|}.\] \citet{Wesselmann12:Implementing} write that ``one may argue that a cut having non-zero coefficients on many (possibly fractional) integer variables is preferable to a cut which consists mostly of continuous variables.''

\section{Omitted results and proofs from Section~\ref{sec:cg_cuts}}\label{app:cg_cuts}

\begin{proof}[Proof of Theorem~\ref{thm:sensitive}]
Without loss of generality, we assume that $n$ is odd. We define the integer program \begin{equation}\begin{array}{ll} \text{maximize} & 0\\
\text{subject to} & 2x[1] + \cdots + 2x[n] = n\\
& \vec{x} \in \{0,1\}^n,
\end{array}\label{eq:Jeroslow}\end{equation}
which is infeasible because $n$ is odd.
\citet{Jeroslow74:Trivial} proved that without the use of cutting planes or heuristics, B\&C will build a tree of size $2^{(n-1)/2}$ before it terminates. Rewriting the equality constraint as $2x[1] + \cdots + 2x[n] \leq n$ and $-2\left(x[1] + \cdots + x[n]\right)\leq -n$, a CG cut defined by the vector $\vec{u} \in \R^2_{\geq 0}$ will have the form $$\left\lfloor 2(u[1] - u[2])\right\rfloor\left(x[1] + \cdots + x[n]\right) \leq \left\lfloor n\left(u[1] - u[2]\right)\right\rfloor.$$

Suppose that $\frac{1}{2} \leq u[1] - u[2] < \frac{n+1}{2n}$. On the left-hand-side of the constraint, $\left\lfloor 2(u[1] - u[2])\right\rfloor = 1$. On the right-hand-side of the constraint, $n\left(u[1] - u[2]\right) < \frac{n+1}{2}$. Since $n$ is odd, $\frac{n+1}{2}$ is an integer, which means that  $\left\lfloor n \left(u[1] - u[2]\right) \right\rfloor \leq \frac{n-1}{2}$. Therefore, the CG cut defined by $\vec{u}$ satisfies the inequality $x[1] + \cdots + x[n] \leq \frac{n-1}{2}$, as illustrated in Figure~\ref{fig:good_cut}. The intersection of this halfspace with the feasible region of the original integer program (Equation~\eqref{eq:Jeroslow}) is empty, so applying this CG cut at the root will cause B\&C to terminate immediately.

Meanwhile, suppose that $\frac{n+1}{2n} \leq u[1] - u[2] < 1$. Then it is still the case that $\left\lfloor 2(u[1] - u[2])\right\rfloor = 1$. Also, $n\left(u[1] - u[2]\right) \geq \frac{n+1}{2}$, which means that $\left\lfloor n \left(u[1] - u[2]\right) \right\rfloor \geq \frac{n+1}{2}$. Therefore, the CG cut defined by $\vec{u}$ dominates the inequality $x[1] + \cdots + x[n] \leq \frac{n+1}{2}$, as illustrated in Figure~\ref{fig:bad_cut}.
The intersection of this halfspace with the feasible region of the original integer program is equal to the integer program's feasible region, so by Jeroslow's result~\citep{Jeroslow74:Trivial}, applying this CG cut at the root will cause B\&C to build a tree of size at least $2^{(n-1)/2}$ before it terminates.
\end{proof}

\begin{proof}[Proof of Lemma~\ref{lem:Wplane}]
Let $\vec{a}_1, \dots, \vec{a}_n \in \R^m$ be the columns of $A$. For $\vec{u}_1\in [0,1]^m,\ldots, \vec{u}_W\in [0,1]^{m+W-1}$, define $\vec{\widetilde{a}}^1_i\in [0,1]^m,\ldots, \vec{\widetilde{a}}^W_i\in[0,1]^{m+W-1}$ for each $i = 1,\ldots, n$ such that $\vec{\widetilde{a}}_i^w$ is the $i$th column of the constraint matrix after applying cuts $\vec{u}_1,\ldots,\vec{u}_{w-1}$. In other words, $\vec{\widetilde{a}}^1_i\in [0,1]^m,\ldots, \vec{\widetilde{a}}^W_i\in[0,1]^{m+W-1}$ are defined recursively as
\begin{align*}
    \vec{\widetilde{a}}^1_i &= \vec{a}_i \\
    \vec{\widetilde{a}}^w_i &= \begin{bmatrix}\vec{\widetilde{a}}^{w-1}_i \\ \vec{u}^T_{w-1}\vec{\widetilde{a}}^{w-1}_i \end{bmatrix}
\end{align*} for $w = 2,\ldots, W$. Similarly, define $\vec{\widetilde{b}}^w$ to be the constraint vector after applying the first $w-1$ cuts:
\begin{align*}
    \vec{\widetilde{b}}^1 &= \vec{b} \\
    \vec{\widetilde{b}}^w &= \begin{bmatrix}\vec{\widetilde{b}}^{w-1} \\ \vec{u}_{w-1}^T\vec{\widetilde{b}}^{w-1}\end{bmatrix}
\end{align*} for $w = 2,\ldots, W$. (These vectors depend on the cut parameters, but we will suppress this dependence for the sake of readability).

We prove this lemma by showing that there are $O(W2^W\norm{A}_{1, 1} + 2^W\norm{\vec{b}}_{1} + nW)$ hypersurfaces determined by polynomials that partition $[0, 1]^m\times \cdots\times [0, 1]^{m+W-1}$ into regions where in any one region $R$, the $W$ cuts 
\begin{align*}
    \sum_{i=1}^n\lf\vec{u}_{1}^{T}\vec{\widetilde{a}}^1_i\rf x[i] &\le\lf\vec{u}_{1}^{T} \vec{\widetilde{b}}^1\rf \\
    &\,\,\,\vdots \\ 
    \sum_{i=1}^n \lf \vec{u}_W^T\vec{\widetilde{a}}^W_i\rf x[i]&\le \lf\vec{u}_W^T\vec{\widetilde{b}}^W\rf
\end{align*}
are invariant across all $(\vec{u}_1,\ldots,\vec{u}_W)\in R$. To this end, let $A_i = \norm{\vec{a}_i}_1$ and $B = \norm{\vec{b}}_1$. For each $w\in \{1,\ldots, W\}$, we claim that $$\lf\vec{u}_w^T\vec{\widetilde{a}}_i^w\rf\in\left[-2^{w-1}A_i, 2^{w-1}A_i\right].$$ We prove this by induction. The base case of $w = 1$ is immediate since $\vec{\widetilde{a}}_i^1 = \vec{a}_i$ and $\vec{u} \in [0,1]^m$. Suppose now that the claim holds for $w$. By the induction hypothesis, $$\norm{\vec{\widetilde{a}}_i^{w+1}}_1 =\norm{\begin{bmatrix}\vec{\widetilde{a}}^{w}_i \\ \vec{u}^T_{w}\vec{\widetilde{a}}^{w}_i \end{bmatrix}}_1 = \norm{\vec{\widetilde{a}}_i^w}_1 + \left|\vec{u}_w^T\vec{\widetilde{a}}_i^w\right|\le 2\norm{\vec{\widetilde{a}}_i^w}_1\le 2^{w}A_i,$$ so $$\lf \vec{u}_{w+1}^T\vec{\widetilde{a}}_i^{w+1}\rf \in\left[-\norm{\vec{\widetilde{a}}_i^{w+1}}_1, \norm{\vec{\widetilde{a}}_i^{w+1}}_1\right]\subseteq [-2^{w}A_i, 2^{w}A_i],$$ as desired. Now, for each integer $k_i\in [-2^{w-1}A_i, 2^{w-1}A_i]$, we have $$\lf\vec{u}_w^T\vec{\widetilde{a}}_i^w\rf = k_i \iff k_i\le \vec{u}_w^T\vec{\widetilde{a}}_i^w < k_i+1.$$ $\vec{u}_w^T\vec{\widetilde{a}}_i^w$ is a polynomial in variables $\vec{u}_1[1],\ldots, \vec{u}_1[m]$, $\vec{u}_2[1],\ldots, \vec{u}_2[m+1],\ldots,$ $\vec{u}_w[1],\ldots, \vec{u}_w[m+w-1]$, for a total of $\le mw + w^2$ variables. Its degree is at most $w$. There are thus a total of $$\sum_{w=1}^W\sum_{i=1}^n (2\cdot 2^{w-1}A_i + 1) = O\left(W2^W\norm{A}_{1, 1} + nW\right)$$ polynomials each of degree at most $W$ plus an additional $\sum_{w=1}^W (2\cdot 2^{w-1}B+1)= O(2^WB + W)$ polynomials of degree at most $W$ corresponding to the hypersurfaces of the form $$k_{n+1}\le \vec{u}_w^T\vec{\widetilde{b}}^w < k_{n+1}+1$$ for each $w$ and each $k_{n+1}\in \{-2^{w-1}B,\ldots, 2^{w-1}B\}$. This yields a total of $O(W2^W\norm{A}_{1, 1} + 2^W\norm{\vec{b}}_{1} + nW)$ polynomials in $\le mW+W^2$ variables of degree $\le W$.
\end{proof}

\begin{proof}[Proof of Theorem~\ref{thm:Wpdim}]
The space of polynomials induced by the $w$th cut, that is, $\{k+\vec{u}_w^T\vec{\widetilde{a}}_i^w : \vec{a}_i\in\R^m, k\in\R\}$, is a vector space of dimension $\le 1+m$. This is because for every $j = 1,\ldots, m$, all monomials that contain a variable $\vec{u}_w[j]$ for some $w$ have the same coefficient (equal to $\vec{a}_i[j]$ for some $1\le i\le n$). Explicit spanning sets are given by the following recursion. For each $j=1,\ldots,m$ define $\vec{\widetilde{u}}_1[j],\ldots, \vec{\widetilde{u}}_W[j]$ recursively as
\begin{align*}
    \vec{\widetilde{u}}_1[j] &= \vec{u}_1[j] \\
    \vec{\widetilde{u}}_w[j] &= \vec{u}_w[j] + \sum_{\ell = 1}^{w-1}\vec{u}_w[m+\ell]\vec{\widetilde{u}}_{\ell}[j]
\end{align*} for $w = 2,\ldots, W$. Then, $\{k+\vec{u}_w^T\vec{\widetilde{a}}_i^w : \vec{a}_i\in\R^m, k\in\R\}$ is contained in $\operatorname{span}\{1, \vec{\widetilde{u}}_w[1],\ldots,\vec{\widetilde{u}}_w[m]\}$. It follows that $$\dim\left(\bigcup_{w=1}^W \{k+\vec{u}_w^T\vec{\widetilde{a}}_i^w : \vec{a}_i\in\R^m, k\in\R\} \right)\le 1+mW.$$ The dual space thus also has dimension $\le 1+mW$. The VC dimension of the family of $0/1$ classifiers induced by a finite-dimensional vector space of functions is at most the dimension of the vector space. Thus, the VC dimension of the set of classifiers induced by the dual space is $\le 1+mW$. Finally, applying the main result of \citet{Balcan21:How} in conjunction with Lemma~\ref{lem:Wplane} gives the desired pseudo-dimension bound. 
\end{proof}

\begin{proof}[Proof of Theorem~\ref{thm:Wkcuts_pdim}]
Applying cuts $\vec{u}^1,\ldots,\vec{u}^k\in[0,1]^m$ simultaneously is equivalent to sequentially applying the cuts $$\vec{u}^1 \in [0,1]^m, \begin{bmatrix}\vec{u}^2 \\ 0 \end{bmatrix}\in[0,1]^{m+1}, \begin{bmatrix}\vec{u}^3 \\ 0 \\ 0 \end{bmatrix}\in [0,1]^{m+2},\ldots,\begin{bmatrix}\vec{u}^k \\ 0 \\ \vdots \\ 0 \end{bmatrix}\in [0,1]^{m+k-1}.$$ Thus, the set in question is a subset of $\left\{f_{\vec{u}_1,\ldots,\vec{u}_{kW}}:\vec{u}_1\in [0,1]^m,\ldots,\vec{u}_{kW}\in[0,1]^{m+kW-1}\right\}$ and has pseudo-dimension $O(mk^2W^2\log(mkW(\alpha+\beta+n)))$ by Theorem~\ref{thm:Wpdim}.
\end{proof}

\subsection{Data-dependent guarantees}\label{app:data}

The \emph{empirical Rademacher complexity}~\cite{Koltchinskii01:Rademacher} of a function class $\cF\subseteq\R^{\cY}$ with respect to $y_1,\ldots, y_N\in\cY$ is the quantity $$\cR_{\cF}(N; y_1,\ldots, y_N)=\E_{\sigma\sim\{-1,1\}^N}\left[\sup_{f\in\cF}\frac{1}{N}\sum_{i=1}^N\sigma_if(y_i)\right].$$ The expected Rademacher complexity $\cR_{\cF}(N)$ of $\cF$ with respect to a distribution $\dist$ on $\cY$ is the quantity $$\cR_{\cF}(N)=\E_{y_1,\ldots, y_N\sim\dist}[\cR_{\cF}(N; y_1,\ldots, y_N)].$$ Rademacher complexity, like pseudo-dimension, is another measure of the intrinsic complexity of the function class $\cF$. Roughly, it measures how well functions in $\cF$ can correlate to random labels. The following uniform convergence guarantee in terms of Rademacher complexity is standard: Let $[-\kappa,\kappa]$ be the range of the functions in $\cF$. Then, for all distributions $\dist$ on $\cY$, with probability at least $1-\delta$ over the draw of $y_1,\ldots, y_N\sim\dist$, for all $f\in\cF$, $\E_{y\sim\dist}[f(y)]-\frac{1}{N}\sum_{i=1}^N f(y_i)\le 2\cR_{\cF}(N)+\kappa\sqrt{\frac{\ln(1/\delta)}{N}}$.

The following result bounds the Rademacher complexity of the class of tree-size functions corresponding to $W$ waves of $k$ CG cuts. The resulting generalization guarantee is more refined than the pseudo-dimension bounds in the main body of the paper. It is in terms of distribution-dependent quantities, and unlike the pseudo-dimension-based guarantees requires no boundedness assumptions on the support of the distribution.

\begin{theorem}\label{thm:rademacher}
    Let $\dist$ be a distribution over integer programs $(\vec{c}, A, \vec{b})$. Let $$\alpha_N = \E_{A_1,\ldots,A_N\sim\dist}\left[\max_{1\le i\le N}\norm{A_i}_{1,1}\right] \quad \text{ and}\quad \beta_N = \E_{\vec{b}_1,\ldots,\vec{b}_N\sim\dist}\left[\max_{1\le i\le N}\norm{\vec{b}}_{1}\right].$$  The expected Rademacher complexity $\cR(N)$ of the class of tree-size functions corresponding to $W$ waves of $k$ Chv\'{a}tal-Gomory cuts with respect to $\dist$ satisfies $$\cR(N)\le O\left(\kappa\sqrt{\frac{mk^2W^2\log(mkW(\alpha_N+\beta_N+n))}{N}}\right)$$ where $\kappa$ is a cap on the size of the tree B\&C is allowed to build.
\end{theorem}

\begin{proof}[Proof of Theorem~\ref{thm:rademacher}]
Let $\cF_{\alpha, \beta}$ denote the class of tree-size functions corresponding to $W$ waves of $k$ CG cuts defined on the domain of integer programs with $\norm{A}_{1,1}\le\alpha$ and $\norm{\vec{b}}_{1}\le\beta$, and let $\cF$ denote the same class of functions without any restrictions on the domain. Applying a classical result due to \citet{Dudley87:Universal}, the empirical Rademacher complexity of $\cF$ with respect to $(\vec{c}_1, A, \vec{b}),\ldots, (\vec{c}_N, A, \vec{b}_N)$ satisfies the bound $$\cR_{\cF}(N; (\vec{c}_1, A, \vec{b}_1),\ldots,(\vec{c}_N, A, \vec{b}_N))\le 60\kappa\sqrt{\frac{\pdim\big(\cF_{\max_i \norm{A_i}_{1, 1}, \max_i\norm{\vec{b}_i}_{1}}\big)}{N}}.$$ Here, $\kappa$ is a bound on the tree-size function as is common in the algorithm configuration literature~\citep{Kleinberg17:Efficiency,Kleinberg19:Procrastinating,Balcan18:Learning}. Taking expectation over the sample, we get \begin{align*}\cR_{\cF}(N) &\le 60\kappa\sqrt{\frac{\E\big[\pdim\big(\cF_{\max_i\norm{A_i}_{1,1}, \max_i\norm{\vec{b}}_{1,1}}\big)\big]}{N}} \\ &\le 60\kappa\sqrt{\frac{\E\big[mk^2W^2\log(mkW(\max_i\norm{A_i}_{1,1}+\max_i\norm{\vec{b}}_1+n))\big]}{N}} \\&\le 60\kappa\sqrt{\frac{mk^2W^2\log(mkW(\alpha_N+\beta_N+n))}{N}}\end{align*} by Theorem~\ref{thm:Wkcuts_pdim} and Jensen's inequality.
\end{proof}

\section{Omitted proofs from Section~\ref{sec:tree_search}}\label{app:tree_search}

\begin{proof}[Proof of Theorem~\ref{thm:main}]
Fix an arbitrary problem instance $x$.
In Claim~\ref{claim:num_hyp}, we prove that for any sequence of actions $\sigma \in \left(\times_{j = 1}^t \left[T_j\right]\right)^{\kappa}$, there is a set of at most $\kappa\sum_{j = 1}^t T_j^2$ halfspaces in $\R^d$
such that Algorithm~\ref{alg:TS} when parameterized by $\vec{\mu} \in \R^d$ will follow the action sequence $\sigma$ if and only if $\vec{\mu}$ lies in the intersection of those halfspaces. Let $\cH_{\sigma}$ be the set of hyperplanes corresponding to those halfspaces, and let $\cH = \bigcup_{\sigma} \cH_{\sigma}$. Since there are at most $\prod_{j = 1}^t T_j^{\kappa}$ action sequences in $\left(\times_{j = 1}^t \left[T_j\right]\right)^{\kappa}$, we know that $|\cH| \leq \kappa\left(\prod_{j = 1}^t T_j^{\kappa}\right)\sum_{j = 1}^t T_j^2$. Moreover, by definition of these halfspaces, we know that for any connected component $C$ of $\R^d \setminus \cH$, across all $\vec{\mu} \in C$, the sequence of actions Algorithm~\ref{alg:TS} follows is invariant. Since the state transitions are deterministic functions of the algorithm's actions, this means that the algorithm's final state is also invariant across all $\vec{\mu} \in C$.
Since the utility function is final-state-constant, this means that $f_{\vec{\mu}}(x)$ is constant across all $\vec{\mu} \in C$.
Therefore, the sample complexity guarantee follows from our general theorem~\citep{Balcan21:How}.
\end{proof}

	\begin{restatable}{claim}{numHyp}\label{claim:num_hyp}
Let $\sigma \in \left(\times_{j = 1}^t \left[T_j\right]\right)^{\kappa}$ be an arbitrary sequence of actions. There are at most $\kappa \sum_{j = 1}^t T_j^2$ halfspaces in $\R^d$ such that Algorithm~\ref{alg:TS} when parameterized by $\vec{\mu} \in \R^d$ will follow the action sequence $\sigma$ if and only if $\vec{\mu}$ lies in the intersection of those halfspaces.
	\end{restatable}
	\begin{proof}
	For each type of action $j \in [t]$, let $k_{j,1}, \dots, k_{j,\kappa} \in [T_j]$ be the sequence of action indices taken over all $\kappa$ rounds. We will prove the claim by
	induction on the step of B\&C. Let $\tree_\tau$ be the state of the B\&C tree after $\tau$ steps. For ease of notation, let $\overline{T} = \sum_{j = 1}^t T_j^2$ be the total number of possible actions squared.
	
	\paragraph{Induction hypothesis.} For a given step $\tau \in [\kappa t]$, let $\kappa_0 \in [\kappa]$ be the index of the current round and $t_0 \in [t]$ be the index of the current action. There are at most $\left(\kappa_0 - 1\right) \overline{T} + \sum_{j = 1}^{t_0} T_j^2$ halfspaces in $\R^{d}$ such that B\&C using the scoring rules $\sum_{i = 1}^{d_j}\mu_{j}[i]\score_{j,i}$ for each action $j \in [t]$ builds the partial search tree $\tree_{\tau}$ if and only if $\left(\vec{\mu}_1, \dots, \vec{\mu}_t\right) \in \R^{d}$ lies in the intersection of those halfspaces.
	
	\paragraph{Base case.}
	In the base case, before
	the first iteration, the set of parameters that will produce the partial search
	tree consisting of just the root is the entire set of parameters, which vacuously is the intersection of zero hyperplanes.
	
	\paragraph{Inductive step.}
	For a given step $\tau \in [\kappa t]$, let $\kappa_0 \in [\kappa]$ be the index of the current round and $t_0 \in [t]$ be the index of the current action. Let $s_{\tau}$ be the state of B\&C at the end of step $\tau$. By the inductive
	hypothesis, we know that there exists a set $\cH$ of at most $\left(\kappa_0 - 1\right) \overline{T} + \sum_{j = 1}^{t_0} T_j^2$ halfspaces such that B\&C using the scoring rules $\sum_{i = 1}^{d_j}\mu_{j}[i]\score_{j,i}$ for each action $j \in [t]$ will be in state $s_{\tau}$ if and only if $\left(\vec{\mu}_1, \dots, \vec{\mu}_t\right) \in \R^{d}$ lies in the intersection of those halfspaces.  Let $\kappa_0' \in [\kappa]$ be the index of the round in step $\tau + 1$ and $t_0' \in [t]$ be the index of the action in step $\tau+ 1$, so \[\left(\kappa_0', t_0'\right) = \begin{cases} \left(\kappa_0, t_0 + 1\right) &\text{if } t_0 < t\\
	\left(\kappa_0 + 1, 1\right) &\text{if }t_0 = t.\end{cases}\] We know B\&C will choose the action $k^* \in \left[T_{t_0'}\right]$ if and only if
	\[\sum_{i = 1}^{d_{t_0'}}\mu_{t_0'}[i]\score_{t_0',i}\left(k^*,s_{\tau}\right) > \max_{k \not= k^*}\sum_{i = 1}^{d_{t_0'}}\mu_{t_0'}[i]\score_{t_0',i}\left(k,s_{\tau}\right). \]
	Since these functions are linear in $\vec{\mu}_{t_0'}$, there are at most $T_{t_0'}^2$ halfspaces defining the region where $k_{t_0', \kappa_0'} = \argmax\sum_{i = 1}^{d_{t_0'}}\mu_{t_0'}[i]\score_{t_0',i}\left(k,s_{\tau}\right)$. Let $\cH'$ be this set of halfspaces. B\&C using the scoring rule $\sum_{i = 1}^{d_{t_0'}}\mu_{t_0'}[i]\score_{t_0',i}$ arrives at state $s_{\tau+1}$ after $\tau+1$ iterations if and only if $\vec{\mu}_{t_0'}$ lies in the intersection of the $\left(\kappa_0' - 1\right) \overline{T} + \sum_{j = 1}^{t_0'} T_j^2$ halfspaces in the set $\cH \cup \cH'$.
\end{proof}

\end{document}